\pdfoutput=1

\documentclass[11pt]{article}

\usepackage{acl}

\usepackage{times}
\usepackage{latexsym}

\usepackage[T1]{fontenc}

\usepackage[utf8]{inputenc}

\usepackage{microtype}

\usepackage{amssymb}
\usepackage{url}
\usepackage{amsmath}
\usepackage{subfig}
\usepackage{graphicx}
\usepackage{graphics}
\usepackage{booktabs}
\usepackage{multirow}
\usepackage{color}
\usepackage{fixltx2e}
\usepackage{enumitem}
\usepackage{pgfplots}
\usepackage{makecell}
\usepackage{comment}
\usepackage{nicefrac}
\usepackage{amsfonts}
\usepackage{bm}
\usepackage{amsthm}

\usepackage{enumitem}
\setlist[itemize,enumerate]{leftmargin=*}

\usepackage[normalem]{ulem}

\DeclareMathOperator*{\argmax}{\text{arg\,max}}

\DeclareMathOperator*{\argtopk}{\text{arg\,top-$k$}}

\newtheorem{theorem}{Theorem}

\newtheorem{proposition}[theorem]{Proposition}

\title{Predicting Attention Sparsity in Transformers}

\author{ \\
        Address line \\ ... \\ Address line}

\author{
Marcos Treviso$^{1,2}$ \quad 
António Góis$^{5}$\thanks{\, Work done at Instituto de Telecomunicações. Correspondence to \texttt{marcos.treviso@tecnico.ulisboa.pt}}  \quad 
Patrick Fernandes$^{1,2,3}$ \\
{\bf Erick Fonseca}$^{6*}$ \quad 
{\bf André F. T. Martins}$^{1,2,4}$ \\
\normalsize $^{1}$Instituto de Telecomunicações, Lisbon, Portugal \\ 
\normalsize $^{2}$Instituto Superior Técnico \& LUMLIS (Lisbon ELLIS Unit), Lisbon, Portugal \\ 
\normalsize $^{3}$Language Technologies Institute, Carnegie Mellon University, Pittsburgh, PA \\
\normalsize $^{4}$Unbabel, Lisbon, Portugal \\ 
\normalsize $^{5}$Mila, Université de Montréal, Canada \\ 
\normalsize $^{6}$Kaufland e-commerce, Cologne, Germany\\
}

\begin{document}
\maketitle
\begin{abstract}
Transformers' quadratic complexity with respect to the input sequence length has motivated a body of work on efficient sparse \emph{approximations} to softmax. 
An alternative path, used by entmax transformers, consists of having built-in \emph{exact} sparse attention; however this approach still requires quadratic computation. 
In this paper, we propose \textit{Sparsefinder}, a simple model trained to \emph{identify} the sparsity pattern of entmax attention before computing it.  
We experiment with three variants of our method, based on distances, quantization, and clustering, on two tasks: machine translation (attention in the decoder) and masked language modeling (encoder-only). 
Our work provides a new angle to study model efficiency by doing extensive analysis of the tradeoff between the sparsity and recall of the predicted attention graph. This allows for detailed comparison between different models along their Pareto curves, important to guide future benchmarks for sparse attention models.
\end{abstract}

\section{Introduction}

Transformer-based architectures have achieved remarkable results in many NLP tasks \citep{vaswani2017attention,devlin-etal-2019-bert,brown2020language}. 
However, they also bring important computational and environmental concerns, caused by their quadratic time and memory computation requirements with respect to the sequence length. 
This comes in addition to the difficulty of interpreting their inner workings, caused by their overparametrization and large number of attention heads. 

There is a large body of work developing ways to ``sparsify'' the computation in transformers, either by imposing local or fixed attention patterns \citep{child2019generating,tay2020sparse,zaheer2020bigbird}, by applying low-rank kernel approximations to softmax \citep{wang2020linformer,choromanski2021rethinking}, or by learning which queries and keys should be grouped together \cite{kitaev-etal-2019-multilingual,DarasSMYRF2020,roy-etal-2021-efficient,wang-etal-2021-cluster}. 
Most of the existing work seeks to \textit{approximate} softmax-based attention by ignoring the (predicted) tails of the distribution, which can lead to performance degradation. 
An exception is transformers with \textbf{entmax-based sparse attention} \citep{correia-etal-2019-adaptively}, a content-based approach which is natively sparse -- this approach has the ability to let each attention head learn from data how sparse it should be, eliminating the need for heuristics or approximations. The disadvantage of this approach is that it still requires a quadratic computation to determine the sparsity pattern, failing to take computational advantage of attention sparsity. 

\begin{figure}[t]
    \centering
    \includegraphics[width=1\columnwidth]{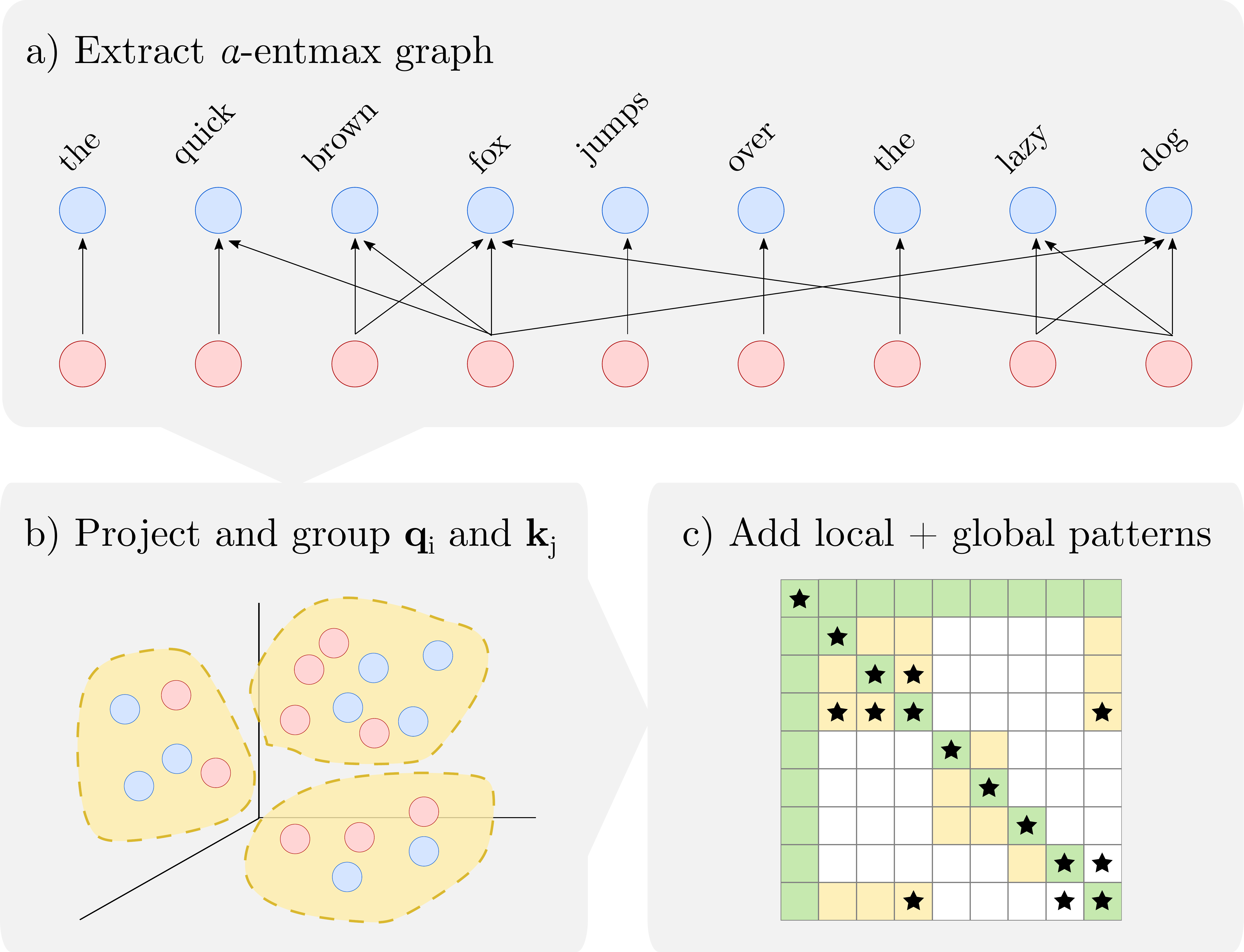}
    \caption{
    (a) Extract sparse attention graphs from a pretrained $\alpha$-entmax transformer; 
    (b) Project query and key vectors to a smaller and appropriated space such that similar points are likely to fall in the same vicinity;
    (c) Additionally, we can combine window and global patterns (green blocks) with the learned pattern (yellow blocks) to increase the recall in recovering ground-truth edges from the sparse graph at the top (starred blocks).
    }
    \label{fig:sparsefinder_overview}
\end{figure}

In this paper, we propose \textbf{Sparsefinder}, which fills the gap above by making entmax attention more efficient (\S\ref{sec:sparsefinder}). Namely, we investigate three methods to predict the sparsity pattern of entmax without having to compute it: one based on metric learning, which is still quadratic but with a better constant (\S\ref{sec:distance-based}), one based on quantization (\S\ref{sec:quantization}), and another based on clustering (\S\ref{sec:clustering}). In all cases, the predictors are trained offline on ground-truth sparse attention graphs from an entmax transformer, seeking high recall in their predicted edges without compromising the total amount of sparsity. 
Figure~\ref{fig:sparsefinder_overview} illustrates our method.

More precisely, to evaluate the effectiveness of our method across different scenarios, we perform experiments on two NLP tasks, encompassing encoder-only and decoder-only configurations: machine translation (MT, \S\ref{sec:experiments_mt}) and masked language modeling (MLM, \S\ref{sec:experiments_qa}), 
doing an extensive analysis of the tradeoff between sparsity and recall (i.e., performance on the attention graph approximation), and sparsity and accuracy (performance on downstream tasks). 
We compare our method with four alternative solutions based on efficient transformers: Longformer \citep{Beltagy2020Longformer}, Bigbird \citep{zaheer2020bigbird}, Reformer \citep{Kitaev2020Reformer}, and Routing Transformer \citep{roy-etal-2021-efficient}, along their entire Pareto curves. 
We complement these experiments by analyzing qualitatively what is selected by the different attention heads at the several layers and represented in different clusters/buckets. %
Overall, our contributions are:%
\footnote{\url{https://github.com/deep-spin/sparsefinder}}

\begin{itemize}
    \item We propose a simple method that exploits learnable sparsity patterns to efficiently compute multi-head attention (\S\ref{sec:sparsefinder}).
    
    \item We do an extensive analysis of the tradeoff between sparsity and recall, and sparsity and accuracy in MT (\S\ref{sec:experiments_mt}) and MLM (\S\ref{sec:experiments_qa}), showing that there is clear room for improvement in the design of efficient transformers.

    \item We qualitatively analyze what is selected by the different attention heads at various layers and represented in different clusters/buckets.
\end{itemize}

\section{Related Work} \label{sec:related_work}

\paragraph{Interpreting multi-head attention.} \label{parag:interpreting_multihead}
Several works analyze the functionalities learned by different attention heads, such as positional and local context patterns \citep{raganato-tiedemann-2018-analysis,voita-etal-2019-analyzing}.  
Building upon prior work on sparse attention mechanisms \citep{peters-etal-2019-sparse},  \citet{correia-etal-2019-adaptively} constrain the attention heads to induce sparse selections individually for each head, bringing interpretability without post-hoc manipulation. Related approaches include the explicit sparse transformer \citep{zhao2019explicit} and rectified linear attention \citep{zhang2021sparse}, which drops the  normalization constraint. 
\citet{raganato-etal-2020-fixed} show that it is possible to fix attention patterns based on previously known behavior (e.g. focusing on previous token) while improving translation quality. However, a procedure that exploits learnable sparsity patterns to accelerate multi-head attention is still missing.  

\paragraph{Low-rank softmax approximations.} \label{parag:low_rank_softmax}

Methods based on low-rank approximation to the softmax such as Linearized Attention  \citep{katharopoulos_et_al_2020}, Linformer \citep{wang2020linformer}, and Performer \citep{choromanski2021rethinking} reduce both speed and memory complexity of the attention mechanism from quadratic to linear, but make interpretability more challenging because the scores are not computed explicitly.
On the other hand, methods that focus on %
inducing sparse patterns provide interpretable alignments and also have performance gains in terms of speed and memory.

\paragraph{Fixed attention patterns.} \label{parag:fixed_attention}

Among fixed pattern methods, Sparse Transformer~\citep{child2019generating} and LongFormer~\citep{Beltagy2020Longformer} attend to fixed positions by using strided/dilated sliding windows. BigBird uses random and two fixed patterns (global and window) to build a block sparse matrix representation~\citep{zaheer2020bigbird}, taking advantage of block matrix operations to accelerate GPU computations.
In contrast, we replace the random pattern with a learned pattern that mimics pretrained $\alpha$-entmax sparse attention graphs.

\paragraph{Learnable attention patterns.} \label{parag:learnable_attention}

Learnable pattern methods usually have to deal with assignment decisions within the multi-head attention mechanism. Clustered Attention~\citep{vyas_et_al_2020} groups query tokens into clusters and computes dot-products only with centroids. Reformer~\citep{Kitaev2020Reformer} and SMYRF~\citep{DarasSMYRF2020} use locality-sensitive hashing to efficiently group tokens in buckets. More similar to our work, Routing Transformer~\citep{roy-etal-2021-efficient} and Cluster-Former~\citep{wang-etal-2021-cluster} cluster queries and keys with online \textit{k}-means and compute dot-products over the top-\textit{k} cluster points. Some queries and keys are discarded due to this filtering, which affects the overall recall of the method (as we show in \S\ref{sec:experiments_mt} and \S\ref{sec:experiments_qa}). The ability of Routing Transformer to benefit from contextual information has been analyzed by \citet{sun-etal-2021-long}. In contrast, Sparsefinder learns to cluster based on sparsity patterns from attention graphs generated by $\alpha$-entmax.

\section{Background}

\subsection{Transformers}

The main component of transformers is the \textbf{multi-head attention} mechanism  
\citep{vaswani2017attention}. 
Given as input a matrix $\mathbf{Q} \in \mathbb{R}^{n \times d}$ containing $d$-dimensional representations for $n$ queries, and matrices $\mathbf{K},\mathbf{V} \in \mathbb{R}^{m \times d}$ for $m$ keys and values, the \textit{scaled dot-product
attention} at a single head is computed  in the following way:
\begin{equation}\label{eq:dotproduct-attention}
    \text{att}(\mathbf{Q}, \mathbf{K}, \mathbf{V}) = 
    \pi
    \underbrace{\Bigg(
        \frac{\mathbf{Q}\mathbf{K}^\top}{\sqrt{d}}
    \Bigg)}_{\mathbf{Z} \in \mathbb{R}^{n \times m}} 
    \mathbf{V} \in \mathbb{R}^{n \times d}.
\end{equation}
The $\pi$ transformation maps rows to distributions, with softmax being the most common choice, 
$\pi(\mathbf{Z})_{ij} = \mathrm{softmax}(\mathbf{z}_i)_j$. 
Multi-head attention is computed by evoking Eq.~\ref{eq:dotproduct-attention} in parallel for each head $h$: 
\begin{equation}\label{eq:multihead-attention} \nonumber
    \text{head}_h(\mathbf{Q}, \mathbf{K}, \mathbf{V}) = \text{att}(\mathbf{Q}\mathbf{W}^Q_h, \mathbf{K}\mathbf{W}^K_h, \mathbf{V}\mathbf{W}^V_h),
\end{equation}
where $\mathbf{W}^Q_h$, $\mathbf{W}^K_h$, $\mathbf{W}^V_h$ are learned linear transformations. 
This way, heads are able to learn specialized phenomena. According to the nature of the input, transformers have three types of multi-head attention mechanism: encoder self-attention (source-to-source), decoder self-attention (target-to-target), and decoder cross-attention (target-to-source). 
While there are no restrictions to which elements can be attended to in the encoder, elements in position $j > i$ in the decoder self-attention are masked  at timestep $i$ (``causal mask'').

\subsection{Extmax Transformers and Learned Sparsity} 

The main computational bottleneck in transformers is the matrix multiplication $\mathbf{Q}\mathbf{K}^\top$ in Eq.~\ref{eq:dotproduct-attention}, which costs $\mathcal{O}(nmd)$ time and can be impractical when $n$ and $m$ are large. 
Many approaches, discussed in \S\ref{sec:related_work}, approximate Eq.~\ref{eq:dotproduct-attention} by ignoring entries far from the main diagonal or computing only some blocks of this matrix, with various heuristics. 
By doing so, the result will be an \textit{approximation} of the softmax attention in Eq.~\ref{eq:dotproduct-attention}. 
This is because the original softmax-based attention is {\it dense}, i.e., it puts \textit{some} probability mass on all tokens -- not only a computational disadvantage, but also  making interpretation harder, as it has been observed that only a small fraction of attention heads capture relevant information \citep{voita-etal-2019-analyzing}. 

An alternative to softmax is the {\bf $\alpha$-entmax transformation}  \citep{peters-etal-2019-sparse,correia-etal-2019-adaptively}, which leads to sparse patterns directly, \textit{without any approximation}: 
\begin{equation}\label{eq:solution_entmax}
    \alpha\text{-entmax}(\mathbf{z}) = [(\alpha - 1)\mathbf{z} - \tau(\mathbf{z})\mathbf{1}]_{+}^{\nicefrac{1}{\alpha-1}},
\end{equation}
where $[\cdot]_{+}$ is the positive part (ReLU) function, and $\tau: \mathbb{R}^n \rightarrow \mathbb{R}$ is a normalizing function satisfying $\sum_j [(\alpha-1)z_j - \tau(\mathbf{z})]_{+}^{\nicefrac{1}{\alpha-1}} = 1$ for any $\mathbf{z}$. That is, entries with score $z_j \leq \nicefrac{\tau(\mathbf{z})}{\alpha-1}$ get exactly zero probability. 
In the limit $\alpha \rightarrow 1$, $\alpha$-entmax recovers the softmax function, while for any value of $\alpha>1$ this transformation can return  sparse probability vectors (as the value of $\alpha$ increases, the induced probability distribution becomes more sparse). When $\alpha=2$, we recover sparsemax \citep{martins2016softmax}. 
In this paper, we use $\alpha=1.5$, which works well in practice and has a specialized fast algorithm %
\citep{peters-etal-2019-sparse}.

Although sparse attention improves interpretability and head diversity when compared to dense alternatives %
\citep{correia-etal-2019-adaptively}, %
the learned sparsity patterns cannot be trivially exploited to reduce the quadratic burden of self-attention, since we still need to compute dot-products between all queries and keys ($\mathbf{Q}\mathbf{K}^\top$) before applying the $\alpha\text{-entmax}$ transformation. In the next section (\S\ref{sec:sparsefinder}), we propose a simple method that learns to \textit{identify} these sparsity patterns beforehand, avoiding the full matrix multiplication. %

\section{Sparsefinder} \label{sec:sparsefinder}

We now propose our method to extract sparse attention graphs and learn where to attend by exploiting a special property of $\alpha$-entmax: \emph{sparse-consistency} (\S\ref{sec:sparse_consistency}). We design three variants of Sparsefinder to that end, based on metric learning (\S\ref{sec:distance-based}), quantization (\S\ref{sec:quantization}), and clustering (\S\ref{sec:clustering}). %

\subsection{Attention graph and sparse-consistency}\label{sec:sparse_consistency}

For each attention head $h$, we define its 
\textbf{attention graph} as 
$\mathcal{G}_h = \{(\mathbf{q}_i, \mathbf{k}_j) \mid p_{i,j} > 0\}$, a bipartite graph connecting query and key pairs $\mathbf{q}_i, \mathbf{k}_j \in \mathbb{R}^d$ for which the $\alpha\text{-entmax}$ probability $p_{i,j}$ is nonzero. An example of attention graph is shown in Figure~\ref{fig:sparsefinder_overview}. 
We denote by $|\mathcal{G}_h|$ the total size of an attention graph, i.e., its number of edges. 
With $\alpha\text{-entmax}$ with $\alpha=1.5$ we typically have $|\mathcal{G}_h| \ll nm$. 
In contrast, softmax attention always leads to a complete graph, $|\mathcal{G}_h| = nm$. 

\paragraph{Problem statement.} 
Our goal is to build a model -- which we call \textit{Sparsefinder} -- that predicts $\hat{\mathcal{G}}_h \approx \mathcal{G}_h$ without having to perform all pairwise comparisons between queries and keys. 
This enables the complexity of evaluating  Eq.~\ref{eq:dotproduct-attention} to be reduced from $\mathcal{O}(nmd)$ to $\mathcal{O}(|\hat{\mathcal{G}}_h|d)$,  effectively taking advantage of the sparsity of $\alpha$-entmax. 
In order to learn such a model, 
we first extract a dataset of sparse attention graphs $\{\mathcal{G}_h\}$ from a pretrained entmax-based transformer, which acts as a teacher. 
Then, the student  learns where to pay attention based on this information. 
This procedure is motivated by the following \textbf{sparse-consistency} property of $\alpha$-entmax: 

\begin{proposition}[Sparse-consistency property] \label{prop:sparse_consistency_property}
Let $\mathbf{b}$ be a binary vector such that $b_j = 1$ if $p_j^\star > 0$, and $b_j = 0$ otherwise. 
For any binary mask vector $\mathbf{m}$ ``dominated'' by $\mathbf{b}$ (i.e.  $\mathbf{m} \odot \mathbf{b} = \mathbf{b}$), we have 
\begin{equation}
    \alpha\text{-entmax}(\mathbf{z}) = \alpha\text{-entmax}(\mathbf{z}|_{\mathbf{m}}),
\end{equation}
where $z_j|_{\mathbf{m}} = z_j$ if $m_j=1$ and $-\infty$ if $m_j=0$. 
\end{proposition}

\begin{proof} See \S\ref{sec:sparse_attention_supp} in the supplemental material.
\end{proof}

This property ensures that, if $\hat{\mathcal{G}}_h$ is such that $\mathcal{G}_h \subseteq \hat{\mathcal{G}}_h$, then we obtain \textit{exactly} the same result as with the original entmax attention. 
Therefore, we are interested in having high recall,
\begin{equation}\label{eq:recall}
    \mathrm{recall}(\hat{\mathcal{G}}_h; \mathcal{G}_h) = \frac{|\hat{\mathcal{G}}_h \cap \mathcal{G}_h|}{|\mathcal{G}_h|},
\end{equation}
meaning that our method is nearly exact, and high sparsity, 
\begin{equation}\label{eq:sparsity}
    \mathrm{sparsity}(\hat{\mathcal{G}}_h) = 1 - \frac{|\hat{\mathcal{G}}_h|}{nm},
\end{equation}
which indicates that computation can be made efficient.\footnote{For the decoder self-attention the denominator in Eq.~\ref{eq:sparsity} becomes $n(n+1)/2$ due to ``causal'' masking.} %
Although a high sparsity may indicate that many computations can be ignored, converting this theoretical result into efficient computation is not trivial and potentially hardware-dependent. In this paper, rather than proposing a practical computational efficient method, we focus on showing that such methods do exist and that they can be designed to outperform fixed and learned pattern methods  while retaining a high amount of sparsity when compared to the ground-truth graph.  %

\paragraph{Our strategies.}

We teach the student model to predict $\hat{\mathcal{G}}_h \approx \mathcal{G}_h$ by taking inspiration from the Reformer model \citep{Kitaev2020Reformer} and the Routing Transformer \citep{roy-etal-2021-efficient}. 
Formally, we define a set of $B$ buckets, $\mathcal{B} = \{1, \ldots, B\}$, and learn functions $f_q, f_k: \mathbb{R}^d \rightarrow 2^{\mathcal{B}} \setminus \{\varnothing\}$, which assign a query or a key to one or more buckets. 
We will discuss in the sequel different design strategies for the functions $f_q, f_k$. 
Given these functions, the predicted graph is:
\begin{equation}
    \hat{\mathcal{G}}_h = \{(\mathbf{q}_i, \mathbf{k}_j) \mid f_q(\mathbf{q}_i) \cap  f_k(\mathbf{k}_j) \neq \varnothing\},
\end{equation}
that is, an edge is predicted between $\mathbf{q}_i$ and $\mathbf{k}_j$ iff they are together in some bucket. 

We present three strategies, based on distance-based pairing (\S\ref{sec:distance-based}), quantization (\S\ref{sec:quantization}) and clustering (\S\ref{sec:clustering}). 
As a first step, all strategies require learning a metric that embeds the graph (projecting queries and keys) into a lower-dimensional space $\mathbb{R}^r$ with $r \ll d$, such that positive query-key pairs are close to each other, and negative pairs are far apart.

\subsection{Learning projections} \label{subsec:learning_projections}

According to the $\alpha$-entmax sparse-consistency property, in order to get a good approximation of $\mathcal{G}_h$, we would like that $f_q$ and $f_k$ produce a graph $\hat{\mathcal{G}}_h$ that maximizes recall, defined in Eq.~\ref{eq:recall}.  
However, maximizing recall in this setting is difficult since we do not have ground-truth bucket assignments. Instead, we recur to a contrastive learning approach by learning projections via negative sampling, which is simpler and more scalable than constrained clustering approaches \citep{wagstaff2001constrained,de2012constrained}.

For each head, we start by projecting the original query and key $\mathbf{q}, \mathbf{k} \in \mathbb{R}^d$ vectors into lower dimensional vectors $\mathbf{q}', \mathbf{k}' \in \mathbb{R}^r$ such that $r \ll d$. 
In practice, we use a simple head-wise linear projection for all queries and keys $g_\theta: \mathbb{R}^d \rightarrow \mathbb{R}^r$. 
To learn the parameters of the projection layer we minimize a hinge loss with margin $\omega$ for each head $h$:
\begin{equation}\label{eq:projection_loss}
    \mathcal{L}_\theta(\mathcal{G}_h) = \Big[ 
    \omega + 
    \|\mathbf{q}' - \mathbf{k}'_{\text{P}}\|^2_2 - 
    \|\mathbf{q}' - \mathbf{k}'_{\text{N}}\|^2_2 
    \Big]_{+},
\end{equation}
where $(\mathbf{q}', \mathbf{k}_{\text{P}}') \in \mathcal{G}_h$ is a positive pair and $(\mathbf{q}', \mathbf{k}_{\text{N}}') \notin \mathcal{G}_h$ is a negative pair sampled uniformly at random. %
In words, we want the distance between a query vector to negative pairs to be larger than the distance to positive pairs by a margin $\omega$. This approach can also be seen as a weakly-supervised learning problem, where the goal is to push dissimilar points away while keeping similar points close to each other \citep{xing2002distance,weinberger2009distance,bellet2015metric}.

\subsection{Distance-based pairing}\label{sec:distance-based} 

To take advantage of the proximity of data points on the embedded space, we first propose a simple method to connect query and key pairs whose Euclidean distance is less than a threshold $t$, i.e. $\hat{\mathcal{G}}_h = \{(\mathbf{q}_i, \mathbf{k}_j) \mid \|\mathbf{q}'_i - \mathbf{k}'_j\|_2 \leq t\}$. Although this method also requires $O(n^2)$ computations, it is more efficient than a vanilla transformer since it reduces computations by a factor of $d/r$ by using the learned projections. This method is also useful to probe the quality of the embedded space learned by the projections, since the recall of our other methods will be contingent on it.

\subsection{Buckets through quantization}\label{sec:quantization}

Our second strategy quantizes each dimension $1, \dots, r$ of the lower-dimensional space into $\beta$ bins, placing the queries and keys into the corresponding buckets ($B = r\beta$ buckets in total). This way, each $\mathbf{q}_i$ and $\mathbf{k}_j$ will be placed in exactly $r$ buckets (one per dimension). If $\mathbf{q}_i$ and $\mathbf{k}_j$ are together in some bucket, Sparsefinder predicts that $(\mathbf{q}_i, \mathbf{k}_j) \in \hat{\mathcal{G}}_h$. Note that for this quantization strategy no learning is needed, only the hyperparameter $\beta$ and the binning strategy need to be chosen. We propose a fixed-size binning strategy: divide each dimension into $\beta$ bins such that all bins have exactly $\lceil n / \beta \rceil$ elements. %
In practice, we append padding symbols to the input to ensure that bins are balanced.

\subsection{Buckets through clustering}\label{sec:clustering}

The clustering strategy uses the low-dimensional projections and runs a clustering algorithm to assign $\mathbf{q}_i$ and $\mathbf{k}_j$ to one or more clusters. In this case, each cluster corresponds to a bucket. In our paper, we employed $k$-means to learn $B$ centroids $\{\mathbf{c}_1, \ldots, \mathbf{c}_B\}$, where each $c_b \in \mathbb{R}^{r}$, over a small portion of the training set. This strategy is similar to the Routing Transformer's online $k$-means \citep{roy-etal-2021-efficient}, but with two key differences: (a) our clustering step is applied offline; (b) we assign points to the top-$k$ closest centroids rather than assigning the closest top-$k$ closest points to each centroid, ensuring that all queries are assigned to a cluster.\footnote{
The difference relies on the dimension on which the top-$k$ operation is applied. 
Routing Transformer applies top-$k$ to the input dimension, possibly leaving some queries unattended, whereas Sparsefinder applies to the centroids dimension, avoiding this problem.} 
At test time, we use the learned centroids to group queries and keys into $k$ clusters each:
\begin{align}
    f_q(\mathbf{q}_i) &= \argtopk_{1 \leq b \leq B} - \|\mathbf{q}_i - \mathbf{c}_b\|^2_2, \\
    f_k(\mathbf{k}_j)  &= \argtopk_{1 \leq b \leq B} - \|\mathbf{k}_j - \mathbf{c}_b\|^2_2,
\end{align}
where the $\argtopk$ operator returns the indices of the $k\textsuperscript{th}$ largest elements.
As in the quantization-based approach, queries and keys will attend to each other, i.e., Sparsefinder predicts  $(\mathbf{q}_i, \mathbf{k}_j) \in \hat{\mathcal{G}}_h$ if they share at least one cluster among the $k$ closest ones. Smaller values of $k$ will induce high sparsity graphs, whereas a larger $k$ is likely to produce a denser graph but with a higher recall.

\subsection{Computational cost}

Let $L$ be the maximum number of elements in a bucket. The time and memory cost of bucketed attention computed through quantization or clustering is $\mathcal{O}(B L^2)$. %
With balanced buckets, we 
get a complexity of $\mathcal{O}(n^{1.5})$ by setting $B = \sqrt{n}$. Although this cost is sub-quadratic, leveraging the sparse structure of $\hat{\mathcal{G}}_h$ in practice is challenging, since it might require specialized hardware or kernels. 
In general, we have $|\hat{\mathcal{G}}_h| = \sum_{b=1}^B n_b m_b \ll nm$, where $n_b$ and $m_b$ are the number of queries and keys in each bucket, since we have small complete bipartite graphs on each bucket. 
Instead of viewing quadratic methods only in light of their performance, we adopt an alternative view of assessing the tradeoff of these methods in terms of sparsity and recall of their approximation $\hat{\mathcal{G}}_h$. This offers a theoretical perspective to the potential performance of each approximation on downstream tasks, helping to find the best approximations for a desired level of sparsity.

\begin{figure*}[t]
    \centering
    \includegraphics[width=1\textwidth]{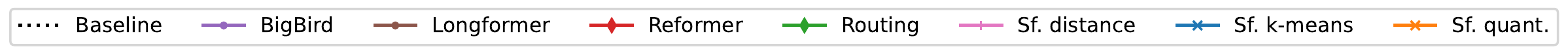}
    \\
    \includegraphics[width=1\textwidth]{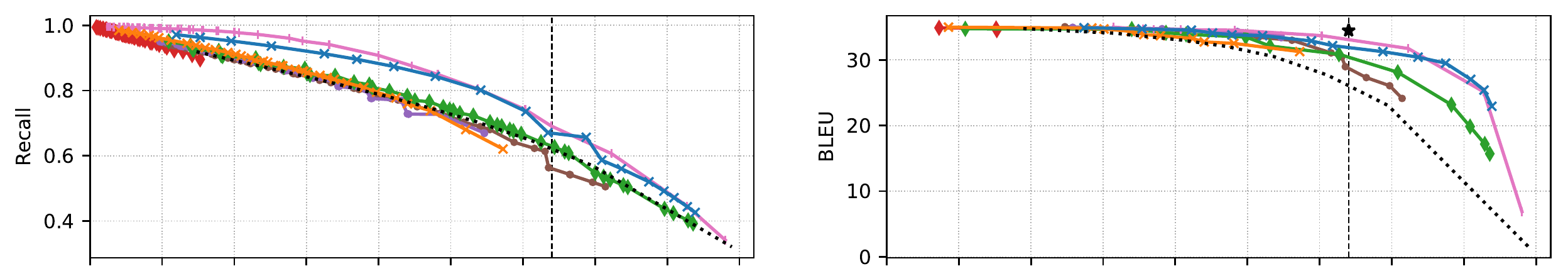} \\
    \includegraphics[width=1\textwidth]{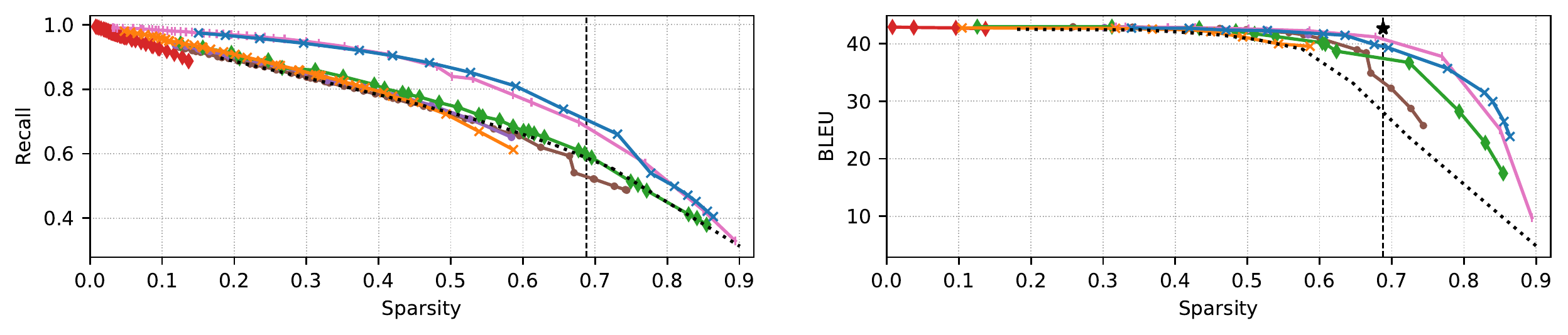} 
    \caption{Sparsity-recall (left) and sparsity-BLEU (right) tradeoff averaged across all layers and heads on IWSLT \textsc{en$\to$de} (top) and \textsc{en$\to$fr} (bottom). The vertical dashed line represents the gold sparsity obtained by the original $\alpha$-entmax transformer (which requires quadratic computation), and the starred marks depict its BLEU score: $34.47$ on \textsc{en$\to$de} and $42.65$ on \textsc{en$\to$fr}.}
    \label{fig:tradeoffs_mt}
\end{figure*} 

\subsection{Combining learned and fixed patterns}

As pointed out in prior work \citep{voita-etal-2019-analyzing}, several attention heads rely strongly in local patterns or prefer to attend to a particular position, more promimently in initial layers.  Therefore, we take inspiration from the Longformer \cite{Beltagy2020Longformer} and BigBird \cite{zaheer2020bigbird} and combine learned sparse patterns with window and global patterns by adding connections in the predicted graph $\hat{\mathcal{G}}_h$ to improve the recall of all methods. Figure~\ref{fig:sparsefinder_overview} illustrates how these patterns are combined in the last step.

\section{Experiments: Machine Translation} \label{sec:experiments_mt}

\paragraph{Setup.} 
We pretrain a \textit{transformer-large} model (6 layers, 16 heads) on the Paracrawl dataset~\citep{espla-etal-2019-paracrawl}. Next, we finetune it with $\alpha$-entmax, fixing $\alpha=1.5$ for all heads, on \textsc{en$\to$de} and \textsc{en$\to$fr} language pairs from IWSLT17 \citep{cettolo2017overview}.
We use the 2011-2014 sets as validation data and the 2015 set as test data. We encode each word using byte pair encoding (BPE, \citealt{sennrich-etal-2016-neural}) with a joint segmentation of 32k merges. 
As \citet{vaswani2017attention}, we finetune our models using the Adam optimizer with an inverse square root learning rate scheduler, with an initial value of $5\times10^{-4}$ and a linear warm-up in the first $4000$ steps.
We evaluate translation quality with sacreBLEU \citep{post-2018-call}. 
Training details, hyperparameters, and data statistics are described in \S\ref{sec:mt_setup}.

\paragraph{Learning projections.}
To learn projections for queries and keys (\S\ref{subsec:learning_projections}),  %
we randomly selected 10K long instances ($n>20$ tokens) from the training set and extracted the $\alpha\text{-entmax}$ attention graphs $\mathcal{G}_h$ from the decoder self-attention for each head. 
This led to an average of 8M and 9M positive pairs ($\mathbf{q}_i, \mathbf{k}_j$) per layer for \textsc{en$\to$de} and \textsc{en$\to$fr}, respectively.
In practice, due to the small number of parameters for each head (only 4,160), a single epoch with Adam was sufficient to optimize the loss in Eq.~\ref{eq:projection_loss}. 
The hyperparameters and the training details for learning projections can be found in \S\ref{sec:mt_setup}.

\paragraph{Pareto-curves.} Using the learned projections, we investigate the recall and the accuracy of all Sparsefinder variants by comparing them with Longformer, BigBird, Reformer, and Routing Transformer. To get a fair comparison, we analyze each method for different levels of sparsity by varying the following hyperparameters:
\begin{itemize}
    \item \textbf{Distance-based methods}: the threshold $t$ within $\{0.5, 1.0, 1.5, 2.0, 2.5, 3.0, 3.5, 4.0, 4.5, 5.0\}$.

    \item \textbf{Bucketing-based methods}: the number of buckets $B$ within $\{2,4,6,8,10,12,16,20\}$.
    
    \item \textbf{Fixed-pattern methods}: the number of random blocks of size 1 within $\{2,4,6,8,10,12,16,20\}$ for BigBird; and the number of random global tokens within $\{2,4,6,8,10,12,16,20\}$ for Longformer.
\end{itemize}
We also add global and local patterns to all methods, varying the window size within $\{0, 1, 3, 5, 7, 9, 11, 15, 19, 23, 27\}$ to get different levels of locality. We further compare all methods with a simple window baseline that only induces the window and global patterns. 
Since all methods exhibit a tradeoff between sparsity and recall/accuracy, we plot the scores obtained by varying the hyperparameters and draw their respective \textbf{Pareto frontier} to see the optimal Pareto-curve. Methods whose points lie below this frontier are said to be \textbf{Pareto-dominated}, meaning that their recall/accuracy cannot be increased without sacrificing sparsity, or vice-versa. Concretely, each point on the curve is measured as a function of the approximation to the ground-truth $\alpha$-entmax attention graph $\mathcal{G}_h$ by replacing it by $\hat{\mathcal{G}}_h$ at test time.

\paragraph{Sparsity-recall tradeoff.}

Pareto-curves for the sparsity-recall tradeoff are shown on the left of Figure~\ref{fig:tradeoffs_mt} for both language pairs. 
Overall, both language pairs have similar trends for all methods. 
Sparsefinder's distance-based and clustering approaches Pareto-dominates the other methods, followed by Routing Transformer. 
Interestingly, Longformer, BigBird, Routing Transformer, and Sparsefinder's bucketing approach perform on par with the  baseline, indicating that a simple local window is a hard baseline to beat.  
Since the LSH attention in Reformer shares queries and keys before hashing, the resultant buckets are also shared for queries and keys, explaining the high recall and the low sparsity of Reformer.  %

\paragraph{Sparsity-accuracy tradeoff.} 
We show the tradeoff between sparsity and BLEU on the right of Figure~\ref{fig:tradeoffs_mt}. For lower levels of sparsity, all methods perform well, close to the full entmax transformer. But as sparsity increases, indicating that only a few computations are necessary, we see that the distance-based and $k$-means variants of Sparsefinder Pareto-dominate other methods, keeping a very high BLEU without abdicating sparsity. 
In particular, Sparsefinder's distance and clustering approaches perform on par with the full entmax transformer when the amount of sparsity is close to the original entmax transformer (around the vertical dashed line). 
Overall, these plots show that methods with a high recall for higher levels of sparsity also tend to have a higher BLEU score.

\paragraph{Learned patterns.} We select some heads and show in Figure~\ref{fig:examples_attention_mt} examples of the pattern learned by our $k$-means variant on \textsc{en$\to$fr}. More examples can be found in \S\ref{sec:attention_plots_supp}. We note that the window pattern is useful to recover local connections. We can see that the $k$-means variant groups more query and key pairs than the actual number of ground-truth edges (left plots). However, due to the sparse-consistency property (right plots), most of these predictions receive zero probability by $\alpha$-entmax, resulting in a very accurate approximation. 

\begin{figure}[t]
    \centering
    \includegraphics[width=0.49\columnwidth]{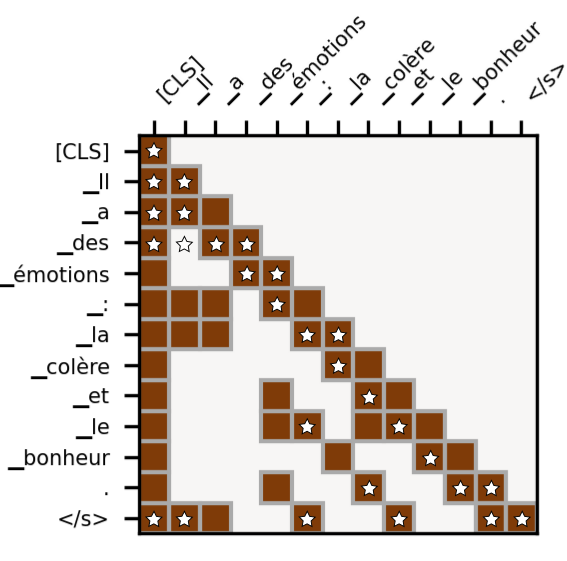}
    \includegraphics[width=0.49\columnwidth]{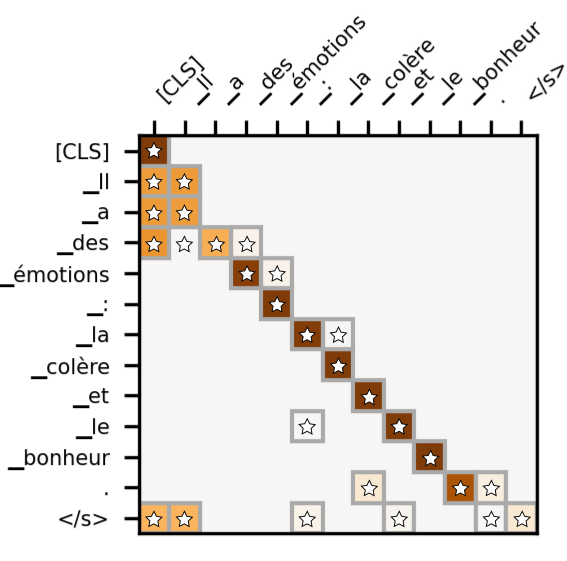}
    \caption{Learned patterns by Sparsefinder $k$-means (left) and the subsequent attention weights (right). Starred blocks represent ground-truth edges.}
    \label{fig:examples_attention_mt}
\end{figure}

\begin{figure*}[t]
    \centering
    \includegraphics[width=1\textwidth]{figs/legend_pareto_accessible.pdf}
    \\
    \includegraphics[width=1\textwidth]{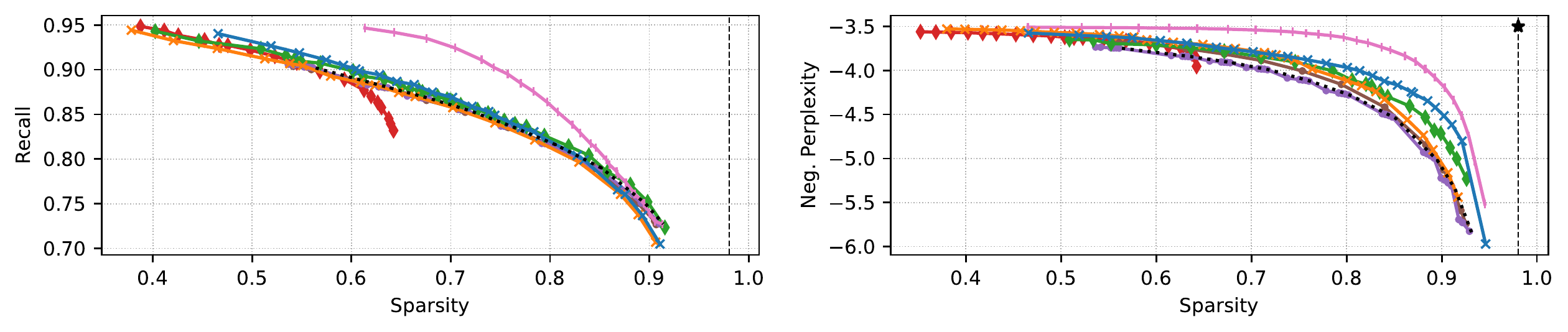}
    \caption{Sparsity-recall and sparsity-(neg-)perplexity tradeoff averaged across all layers and heads on WikiText-103. The vertical dashed line represents the gold sparsity obtained by the full $\alpha$-entmax transformer.}
    \label{fig:tradeoffs_mlm}
\end{figure*}

\section{Experiments: Masked LM} \label{sec:experiments_qa}

\paragraph{Setup.} Following \citet{Beltagy2020Longformer}, we initialize our model from a pretrained RoBERTa checkpoint. We use the {\tt roberta-base} model from Huggingface's transformers library, with 12 layers and 12 heads.\footnote{\url{https://huggingface.co/roberta-base}} We finetune on WikiText-103 \citep{merity2016pointer}, replacing softmax by $\alpha\text{-entmax}$ with $\alpha=1.5$ for all heads. Training details, model hyperparameters, and data statistics can be found in~\S\ref{sec:mlm_setup}.

\paragraph{Learning projections.} 
As done for MT experiments, we learn to project keys and queries from the original 64 dimensions into $r=4$ dimensions. 
To this end, we use 1K random samples from the training set, each with length of 512, keeping half for validation. We extract the $\alpha\text{-entmax}$ attention graphs $\mathcal{G}_h$ but from the encoder self-attention of each head, leading to an average of 3M positive pairs per layer. Due to the small number of learnable parameters for each head ($256$), training was done with Adam for one epoch.

\paragraph{Results.}
Our full transformer trained with $\alpha$-entmax achieved a perplexity score of $3.5004$ with an overall sparsity of $0.9804$ on WikiText-103. As in sentence-level MT experiments, we measure the sparsity-recall and the sparsity-perplexity tradeoff via the change of $\mathcal{G}_h$ with $\hat{\mathcal{G}}_h$ at test time. Moreover, since MLM has longer inputs, we increased the range of the window pattern to $\{31, 41, 51, 75, 101, 125, 151, 175, 201, 251\}$. 

We show in Figure~\ref{fig:tradeoffs_mlm} the Pareto curves for the tradeoff between sparsity and recall (left), and the tradeoff between sparsity and perplexity (right). The curves for the sparsity-recall tradeoff are similar to the ones found in MT experiments, with the distance-based method outperforming all methods, followed by the $k$-means variant of Sparsefinder and Routing Transformer. 
In terms of perplexity, our distance-based approach also Pareto-dominates other methods, followed by our clustering variant and Routing Transformer.
As in the MT experiments, the window baseline yields a similar sparsity-recall curve to other approaches, reinforcing the importance of local patterns.
Although the distance-based method requires a quadratic number of computations, it reduces them by a factor of $d/r=64/4=16$, as described in \S\ref{sec:distance-based}, and achieves better recall and perplexity than any other tested method. This finding indicates clear room for improvement in designing efficient attention methods that have a better tradeoff between efficiency and accuracy than existing approaches. 

\paragraph{Learned patterns.} In Figure~\ref{fig:examples_attention_graph_mlm} we show  Sparsefinder $k$-means' predicted attention graphs for a specific attention head that originally learned to focus on coreference tokens. We can see that the pattern induced by Sparsefinder keeps the behavior of attending to coreferences. Concretely, our method achieves a high recall score ($\sim80\%$) with a high sparsity rate ($\sim75\%$) on this attention head. %

\begin{figure}[!htb]
    \centering
    \includegraphics[width=0.49\columnwidth]{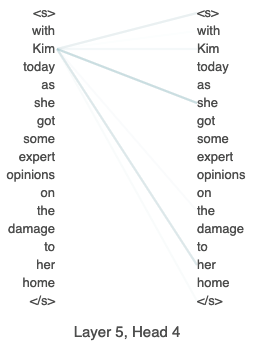}
    \includegraphics[width=0.49\columnwidth]{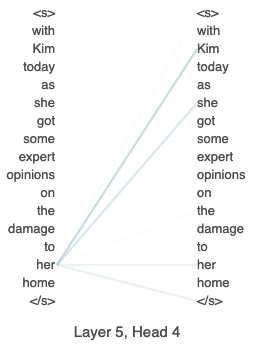}
    \caption{Attention pattern learned by Sparsefinder $k$-means that focus on coreference tokens.}
    \label{fig:examples_attention_graph_mlm}
\end{figure}

\paragraph{Cluster analysis.} To understand what is represented in each cluster learned by Sparsefinder $k$-means, we run the following experiment: we obtain POS tags using spaCy,\footnote{\url{https://spacy.io/}} and calculate the distribution of each tag over clusters for all heads. We show an example in Figure~\ref{fig:cluster_pos}, where Sparsefinder learned a cluster that makes verbs and nouns attend to themselves, and additionally to most auxiliary verbs.

\begin{figure}[!htb]
    \centering
    \includegraphics[width=\columnwidth]{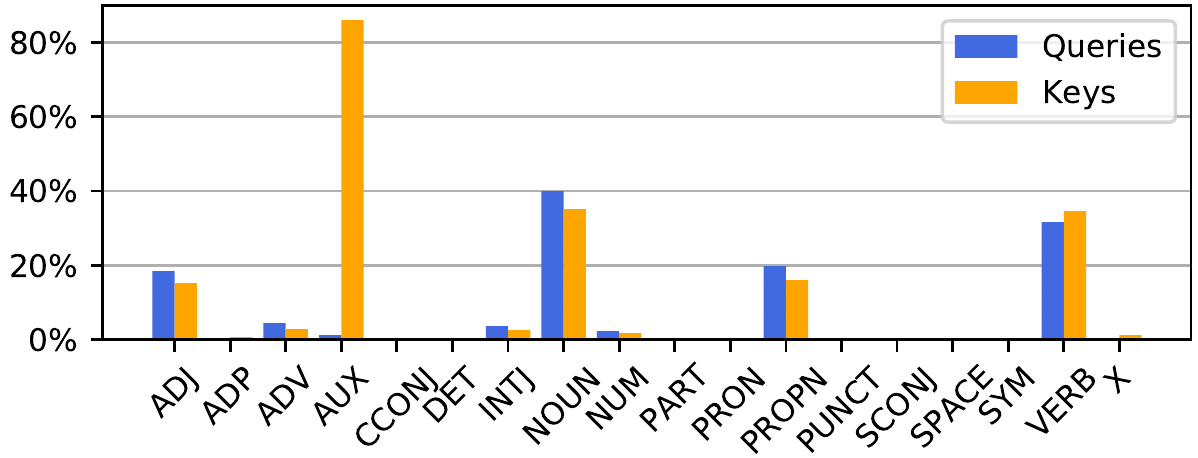}
    \caption{
    Percentage of POS tags assigned to a given cluster on the entire Wikitext 103 validation set.
    }
    \label{fig:cluster_pos}
\end{figure}

\subsection{Efficient Sparsefinder}

\begin{figure*}[t]
    \centering
    \includegraphics[width=0.94\textwidth]{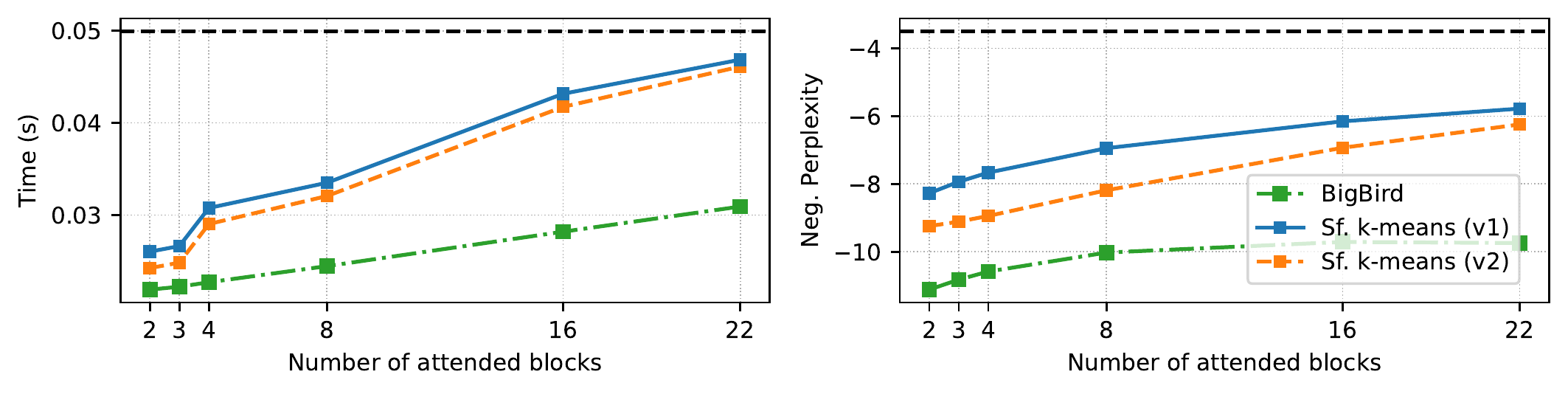}
    \caption{Comparison of Sparsefinder and BigBird in terms of running time and (negative) perplexity by varying the number of attended blocks. The black dashed line represents the results obtained by the full $\alpha$-entmax transformer.}
    \label{fig:sparsefinder_vs_bigbird_time}
\end{figure*}

We now turn to the question of making Sparsefinder efficient in practice. Before we proceed, we note that comparison between methods usually depends on the specific implementation used, which influences the measurements and can also require specialized hardware. 
This leaves BigBird and Routing Transformer as the only models we can compare with in practice: Reformer includes other optimizations that are not part of the attention mechanism, and Longformer is based on CUDA kernels, specialized for fast computation. Lastly, the strategy used in Routing Transformer is incorporated in Sparsefinder (v2), where we use Sparsefinder's centroids with Routing Transformer top-$k$ strategy.
In order to make Sparsefinder more efficient, we adopt the key strategy of BigBird: work with contiguous chunks rather than single tokens, creating blocks in the attention matrix. More precisely, we learn projections over chunked tokens following Equation~\ref{eq:projection_loss}, where $(\mathbf{q}', \mathbf{k}'_P)$ is a positive pair if any token inside the chunk is part of a positive pair of the original $\alpha$-entmax graph, and similarly, a pair $(\mathbf{q}', \mathbf{k}'_N)$ is negative if all tokens inside the chunk are negative. Thus, given a block/chunk size $z$, the size of the dense attention graph reduces from $|\mathcal{G}_h| = nm$ to $|\mathcal{G}_h| = \lceil nm / z^2 \rceil$ (with zero-padding).

\paragraph{Implementation.} In order to be comparable to BigBird, we implement a routine that caps the maximum number of attended blocks in Sparsefinder, analogous to the number of random blocks used in BigBird. We propose two variants: \textbf{(v1)} computes dot-products between all chunked vector projections and then returns the top-$k$ blocks, and \textbf{(v2)} selects the top-$k$ blocks closest to the learned centroids and computes dot-products for these blocks. The first variant is more costly, yet it may lead to a more robust selection, whereas the second variant resembles Routing Transformer's top-$k$ strategy.

\paragraph{Results.} We measure the clock-time of the MLM model evaluated on 500 examples with a batch size of 8. We vary the number of attended blocks within \{2, 3, 4, 8, 16, 22 $\approx \sqrt{n}$\}, the block size in \{2, 4, 8, 16\}, and compute perplexity for values of $B$ (number of clusters) within \{2, 4, 8, 12, 16, 20\}. We use a window size of $3$ in all experiments to capture the controlled hyperparameters' impact better. Figure~\ref{fig:sparsefinder_vs_bigbird_time} shows plots by averaging runs with different block sizes and number of clusters. 
As expected, using a lower number of attended blocks leads to improvements in terms of running time, yet all models perform poorly on the MLM task. As we increase the number of blocks, we can see both a boost in terms of MLM performance and an increased running time. By comparing Sparsefinder and BigBird, we notice that BigBird is faster than Sparsefinder, but increasing the number of attended (random) blocks in BigBird does not lead to significant improvements on the real task. In contrast, both versions of Sparsefinder can improve the MLM performance while still being faster than a regular $\alpha$-entmax transformer. In particular, by attending to only 2 blocks, Sparsefinder is able to achieve a better MLM score than BigBird with 22 random blocks while still being faster than it. Plots for each block size can be found in \S\ref{sec:efficient_sparsefinder_supp}.

\section{Conclusions}

We proposed Sparsefinder, a  method to identify the sparsity pattern of entmax-based transformers while avoiding full computation of the score matrix.  
Our method learns a low-dimensional projection of queries and keys with a contrastive objective, %
and comes with three variants: distance, quantization, and clustering-based. We compared these variants against competing approaches on two tasks: machine translation and masked language modeling. We obtained favorable sparsity-recall and sparsity-accuracy tradeoff curves.  Our theoretical sparsity provides a lower bound for how much computational sparsity can be achieved, and may guide future research on efficient transformers. %
Finally, we proposed a simple extension of Sparsefinder that resembles the block-based attention of BigBird by learning projections of chunked tokens, which exhibits a promising direction in terms of the trade-off of learnable sparsity with computation time and accuracy.

\section*{Acknowledgments}
This work was supported by the European Research Council (ERC StG DeepSPIN 758969), P2020 project MAIA (LISBOA-01-0247-FEDER045909), and Fundação para a Ciência e Tecnologia through project PTDC/CCI-INF/4703/2021 (PRELUNA) and contract UIDB/50008/2020.

\bibliography{anthology,acl}
\bibliographystyle{acl_natbib}

\appendix

\clearpage
\onecolumn
\appendix

\section{Sparse Attention}\label{sec:sparse_attention_supp}

A natural way to get a sparse attention distribution is by using the {\bf sparsemax transformation} \citep{martins2016softmax},
which computes an Euclidean projection of the score vector onto the probability simplex $\triangle^n := \{\mathbf{p} \in \mathbb{R}^n \mid \mathbf{p}\ge \mathbf{0}, \,\, \mathbf{1}^\top\mathbf{p} = 1\}$, or, more generally, the {\bf $\alpha$-entmax transformation} \citep{peters-etal-2019-sparse}: 
    \begin{equation}\label{eq:entmax}
        \alpha\text{-entmax}(\mathbf{z}) := \argmax_{\mathbf{p} \in \triangle^{n}} \mathbf{p}^\top \mathbf{z} + H_\alpha(\mathbf{p}),
    \end{equation}
where $H_\alpha$ is a generalization of the Shannon and Gini entropies proposed by \citet{Tsallis1988},  parametrized by a scalar $\alpha\ge 1$: 
\begin{equation}\label{eq:tsallis}
    H_\alpha (\mathbf{p}) := \begin{cases}
                                                \frac{1}{\alpha(\alpha-1)}\sum_j(p_j-p_j^\alpha), &  \alpha \neq 1\\
                                                -\sum_j p_j \log p_j, & \alpha=1.
                                            \end{cases}
\end{equation}
Setting $\alpha=1$ recovers the softmax function, while for any value of $\alpha>1$ this transformation can return a sparse probability vector. %
Letting $\alpha=2$, we recover sparsemax. A popular choice is $\alpha=1.5$, which has been successfully used  in machine translation and morphological inflection applications \citep{peters-etal-2019-sparse,correia-etal-2019-adaptively}.

\paragraph{Proof to Proposition~\ref{prop:sparse_consistency_property}.}
\begin{proof}

From the definition of $\mathbf{z}|_\mathbf{m}$ and from Eq.~\ref{eq:solution_entmax}, we have that
    \begin{equation}\label{eq:proof_ineq}
        \begin{cases} 
            z_j|_\mathbf{m} = z_j >  \frac{\tau(\mathbf{z})}{\alpha-1}  & \text{if } p_j^* > 0 \\
            z_j|_\mathbf{m} \le z_j \le \frac{\tau(\mathbf{z})}{\alpha-1} & \text{if } p_j^* = 0.
        \end{cases}
    \end{equation}
We first prove that $\tau(\mathbf{z}|_\mathbf{m}) = \tau(\mathbf{z})$. 
From the definition of $\tau(\mathbf{z})$ we have that $\sum_j [(\alpha-1)z_j - \tau(\mathbf{z})]_{+}^{\nicefrac{1}{\alpha-1}} = 1$. Plugging the (in)equalities from Eq.~\ref{eq:proof_ineq}, we thus have
\begin{align}\label{eq:proof_tau}
    1 &= \sum_j [(\alpha-1)z_j - \tau(\mathbf{z})]_{+}^{\nicefrac{1}{\alpha-1}} = \sum_j [(\alpha-1)z_j|_\mathbf{m} - \tau(\mathbf{z})]_{+}^{\nicefrac{1}{\alpha-1}}.
\end{align}
Since $\tau(\mathbf{z})$ satisfies the second equation -- which is the condition that defines  $\tau(\mathbf{z}|_\mathbf{m})$ -- we thus conclude that $\tau(\mathbf{z}|_\mathbf{m}) = \tau(\mathbf{z})$. 
Combining the results in Eqs.~\ref{eq:proof_ineq}--\ref{eq:proof_tau}, we see that the supports of $\alpha\text{-entmax}(\mathbf{z})$ and $\alpha\text{-entmax}(\mathbf{z}|_\mathbf{m})$ are the same and so are the thresholds $\tau$, and therefore from Eq.~\ref{eq:solution_entmax} we conclude that $\alpha\text{-entmax}(\mathbf{z}|_\mathbf{m}) = \alpha\text{-entmax}(\mathbf{z})$.
\end{proof}

\section{Computing infrastructure}

Our infrastructure consists of 4 machines with the specifications shown in Table~\ref{table:computing_infrastructure}. The machines were used interchangeably, and all experiments were executed in a single GPU. Despite having machines with different specifications, we did not observe large differences in the execution time of our models across different machines. 

\begin{table}[!htb]
    \small
    \begin{center}
    \begin{tabular}{l ll}
        \toprule
        \sc \# & \sc GPU & \sc CPU  \\
        \midrule
        1.   & 4 $\times$ Titan Xp - 12GB           & 16 $\times$ AMD Ryzen 1950X @ 3.40GHz - 128GB \\
        2.   & 4 $\times$ GTX 1080 Ti - 12GB        & 8 $\times$ Intel i7-9800X @ 3.80GHz - 128GB \\
        3.   & 3 $\times$ RTX 2080 Ti - 12GB        & 12 $\times$ AMD Ryzen 2920X @ 3.50GHz - 128GB \\
        4.   & 3 $\times$ RTX 2080 Ti - 12GB        & 12 $\times$ AMD Ryzen 2920X @ 3.50GHz - 128GB \\
        \bottomrule
    \end{tabular}
    \end{center}
    \caption{Computing infrastructure.} 
    \label{table:computing_infrastructure}
\end{table}

\section{Machine Translation}\label{sec:mt_setup}

\subsection{Setup}

\paragraph{Data.}
Statistics for all datasets used in MT experiments can be found below in Table~\ref{table:datasets_mt}.
\begin{table}[!htb]
    \small
    \begin{center}
    \begin{tabular}{lccc}
        \toprule
        \sc Dataset & \sc \# train & \sc \# test & \sc Avg. sentence length \\
        \midrule
        IWSLT17 (\textsc{en$\to$de})       & 206K          & 1080          & 20 \textcolor{gray}{±14} \,\,\,/\,\,\, 19 \textcolor{gray}{±13} \\
        IWSLT17 (\textsc{en$\to$fr})       & 233K          & 1210          & 20 \textcolor{gray}{±14} \,\,\,/\,\,\, 21 \textcolor{gray}{±15} \\
        \bottomrule
    \end{tabular}
    \end{center}
    \caption{Statistics for MT datasets. 
    }
    \label{table:datasets_mt}
\end{table}

\paragraph{Training and Model.}
We replicated the sentence-level model of \citet{fernandes21acl} 
with the exception that we used $\alpha$-entmax with $\alpha=1.5$ instead of softmax in all attention heads and layers. Table~\ref{tab:hyperparams_nmt} shows some architecture (transformer large) and training hyperparameters used for MT experiments.  We refer to the original work of \citet{fernandes21acl} for more training details.

\begin{table}[!htb]
    \centering
    \small
    \begin{tabular}{ll}
        \toprule
        \sc Hyperparam. & \sc Value  \\
        \midrule
        Hidden size        & 1024     \\
        Feedforward size        & 4096     \\
        Number of layers        & 6     \\
        Number of heads        & 16     \\
        Attention mapping $\pi$     & $1.5$-entmax \\
        Optimizer                   & Adam  \\
        Number of epochs            & 20    \\
        Early stopping patience     & 10     \\
        Learning rate               & 0.0005 \\
        Scheduling                  & Inverse square root \\
        Linear warm-up steps        & 4000 \\
        Dropout                     & 0.3 \\
        CoWord dropout              & 0.1 \\
        
        Beam size                   & 5 \\
        \bottomrule
    \end{tabular}
    \caption{Hyperparmeters for neural machine translation models.}
    \label{tab:hyperparams_nmt}
\end{table}

\subsection{Projections setup}\label{sec:nmt_projections_setup}

\paragraph{Data.}
Statistics for the subsets of IWSLT used in the projection analysis can be found below in Table~\ref{table:datasets_proj_mt}.
\begin{table}[!htb]
    \small
    \begin{center}
    \begin{tabular}{l ccc c ccc}
        \toprule
        & \multicolumn{3}{c}{\sc Train} & & \multicolumn{3}{c}{\sc Validation} \\
        \cmidrule{2-4} \cmidrule{6-8}
        \sc Pair & \sc \# sent. & \sc \# pos. pairs & \sc Avg. sent. length & & \sc \# sent. & \sc \# pos. pairs & \sc Avg. sent. length \\
        \midrule
        \textsc{en$\to$de}       & 9K  & 8M \textcolor{gray}{±1M}  & 35 \textcolor{gray}{±16}  & & 1K  & 330K \textcolor{gray}{±56K}  & 36 \textcolor{gray}{±17} \\
        \textsc{en$\to$fr}       & 9K  & 9M \textcolor{gray}{±1M}  & 37 \textcolor{gray}{±17}  & & 1K  & 334K \textcolor{gray}{±58K}  & 37 \textcolor{gray}{±16} \\
        \bottomrule
    \end{tabular}
    \end{center}
    \caption{Statistics for subsets of IWSLT used for training and evaluating projections.
    }
    \label{table:datasets_proj_mt}
\end{table}

\paragraph{Training.} After extracting the $\alpha$-entmax graphs, we optimize the learnable parameters of Equation~\ref{eq:projection_loss} with Adam over a single epoch. Moreover, we used the $k$-means implementation from scikit-learn~\citep{scikit-learn} for our clustering-based approach. The hyperparameters used both for training the projections and for clustering with $k$-means are shown in Table~\ref{tab:hyperparams_projs_and_kmeans}.

\begin{table}[!htb]
    \centering
    \small
    \begin{tabular}{ll}
        \toprule
        \sc Hyperparam. & \sc Value  \\
        \midrule
        Projection dim. $r$         & 4     \\
        Loss margin $\omega$        & 1.0 \\
        
        Batch size                  & 16    \\
        Optimizer                   & Adam      \\
        Number of epochs            & 1     \\
        Learning rate               & 0.01     \\
        $\ell_2$ regularization     & 0     \\
        
        $k$-means init              & $k$-means++ \\
        $k$-means max num. inits    & 10 \\
        $k$-means max iters         & 300  \\
        \bottomrule
    \end{tabular}
    \caption{Hyperparmeters for MT projections.}
    \label{tab:hyperparams_projs_and_kmeans}
\end{table}

\paragraph{Projection analysis.}

We compare Sparsefinder, varying $B \in \{2, 4, 6, 8, 10, 12\}$ for bucket-based methods, and $t \in \{0.5, 1.0, 1.5, 2.0, 2.5\}$ for the distance-based variant, with the following methods:

\begin{itemize}
    \item \textbf{Window baseline:} connect all query and key pairs within a sliding window of size $w \in \{0, 1, 3, 5, 7, 9, 11, 15, 19, 23, 27\}$.

    \item \textbf{Learnable patterns:} Reformer by varying the number of buckets within $\{2, 4, 6, 8, 10, 12\}$; Routing transformer by varying the number of clusters within~$c \in \{2, 4, 6, 8, 10\}$ with top-$k$ set to $\lceil n / c \rceil$ (i.e. balanced clusters). 
    
    \item \textbf{Fixed patterns:} BigBird by varying the number of random blocks within $\{2, 4, 6, 8, 10\}$ with a block size of $1$; Longformer by varying the number of random global tokens within $\{4, 8, 12, 16, 20\}$. 
   
\end{itemize}

\paragraph{Sparsity-recall tradeoff per layer and head.} Plots are shown in Figures \ref{fig:sparsity_recall_tradeoff_mt_ende_projections_supp2} and \ref{fig:sparsity_recall_tradeoff_mt_enfr_projections_supp2} for \textsc{en$\to$de} and \textsc{en$\to$fr}, respectively.

\begin{figure}[!htb]
    \centering
    \includegraphics[width=\textwidth]{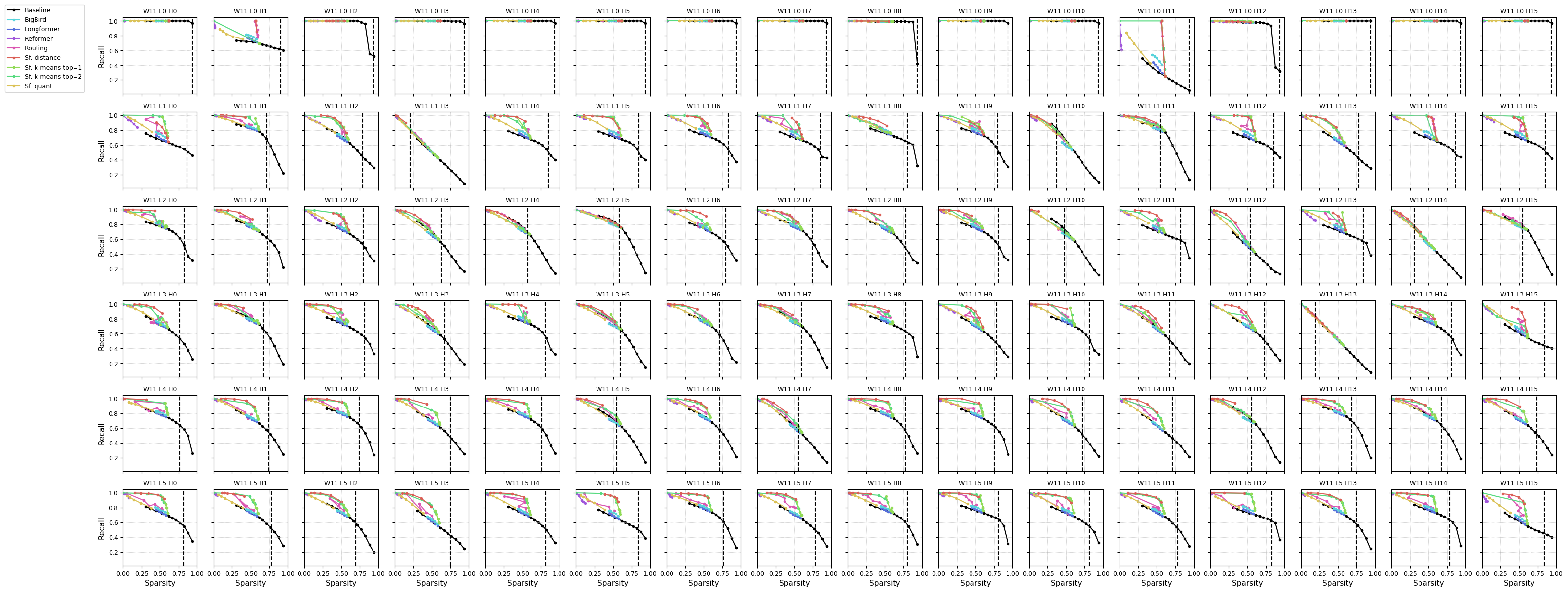}
    \caption{Sparsity-recall tradeoffs with a fixed window pattern of size 11 for \textsc{en$\to$de}.}
    \label{fig:sparsity_recall_tradeoff_mt_ende_projections_supp2}
\end{figure}

\begin{figure}[!htb]
    \centering
    \includegraphics[width=\textwidth]{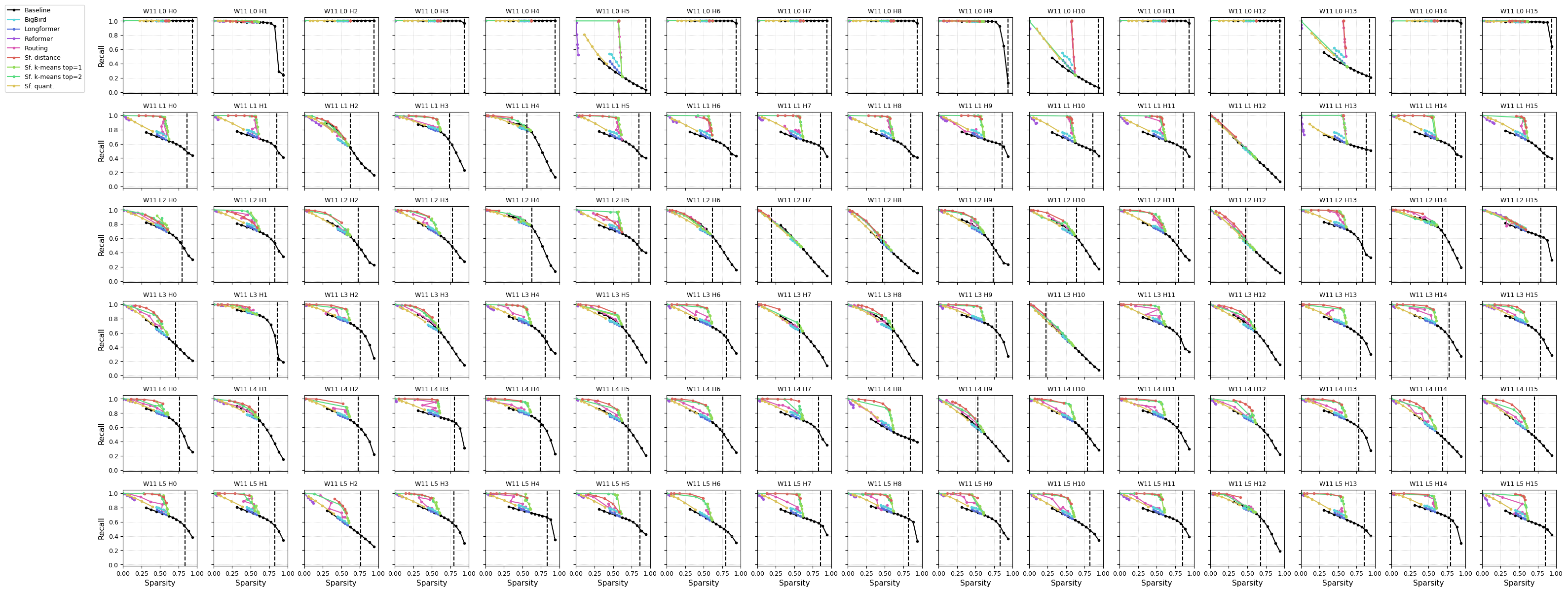}
    \caption{Sparsity-recall tradeoffs with a fixed window pattern of size 11 for \textsc{en$\to$fr}.}
    \label{fig:sparsity_recall_tradeoff_mt_enfr_projections_supp2}
\end{figure}

\section{Masked Language Modeling} \label{sec:mlm_setup}

\subsection{Setup}

\paragraph{Data and model.} In order to have a transformer model trained with $\alpha$-entmax, we finetuned RoBERTa-Base~\citep{liu2019roberta} on WikiText-103 \citep{merity2016pointer} over 3000 steps with Adam (learning rate of $3\times10^{-5}$).  To mimic the finetuning approach adopted by Longformer, we employed a batch size of 2 by accumulating gradients over 32 steps due to GPU memory constraints. 
Table~\ref{tab:hyperparams_mlm} shows some architecture (transformer large) and training hyperparameters used for MT experiments.  
We refer to the original work of \citet{liu2019roberta} for more architecture details.

\begin{table}[!htb]
    \centering
    \small
    \begin{tabular}{ll}
        \toprule
        \sc Hyperparam. & \sc Value  \\
        \midrule
        Hidden size             & 64     \\
        Feedforward size        & 3072     \\
        Max input length        & 514 \\
        Number of layers        & 12     \\
        Number of heads        & 12     \\
        Attention mapping $\pi$     & $1.5$-entmax \\
        Optimizer                   & Adam  \\
        Number of steps             & 3000    \\
        Learning rate               & 0.00003 \\
        \bottomrule
    \end{tabular}
    \caption{Hyperparmeters for masked language modeling models.}
    \label{tab:hyperparams_mlm}
\end{table}

\subsection{Projections setup} \label{sec:mlm_projections_setup}
 
\paragraph{Data and training.} The subset used for Masked LM projections experiments contains 500 instances for training and 500 instances for validation. Moreover, all instances have a sentence length of 512 tokens. We got 3M (±1M) positive pairs for training and 2.5M (±1M) for validation. The hyperparameters for Masked LM are the same as the ones used in the MT experiments, shown in Table~\ref{tab:hyperparams_projs_and_kmeans}.

\paragraph{Projection analysis.}

We perform the same analysis as in MT, but now we vary the window size of the baseline within \{0, 1, 3, 7, 11, 25, 31, 41, 51, 75, 101, 125, 151, 175, 201, 251, 301, 351, 401, 451, 501, 512\}. 

\paragraph{Sparsity-recall tradeoff per layer and head.} Plots are shown next in Figure~\ref{fig:sparsity_recall_tradeoff_mlm_projections_supp2}.

\begin{figure}[!htb]
    \centering
    \includegraphics[width=\textwidth]{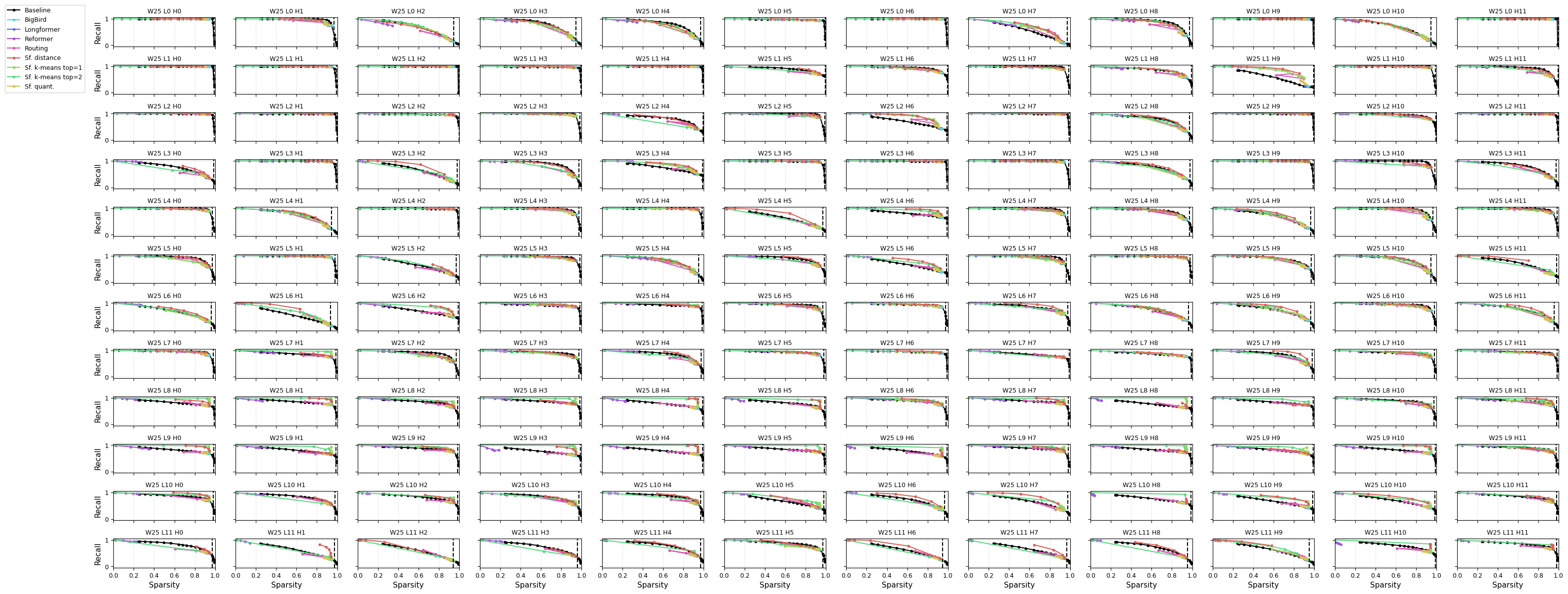}
    \caption{Sparsity-recall tradeoffs with a fixed window pattern of size 25 for MLM.}
    \label{fig:sparsity_recall_tradeoff_mlm_projections_supp2}
\end{figure}

\section{Attention plots} \label{sec:attention_plots_supp}

Examples of attention maps can be seen in Figure~\ref{fig:ex1_attention_map_supp} and \ref{fig:ex2_attention_map_supp}.

\begin{figure}[!htb]
    \centering
    \includegraphics[width=0.49\columnwidth]{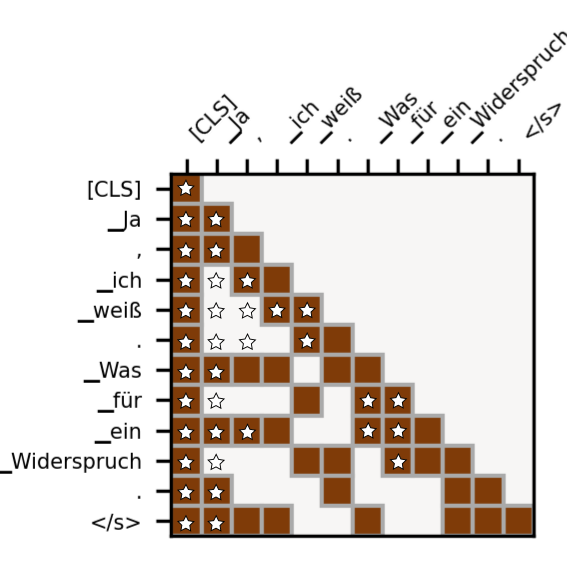}
    \includegraphics[width=0.49\columnwidth]{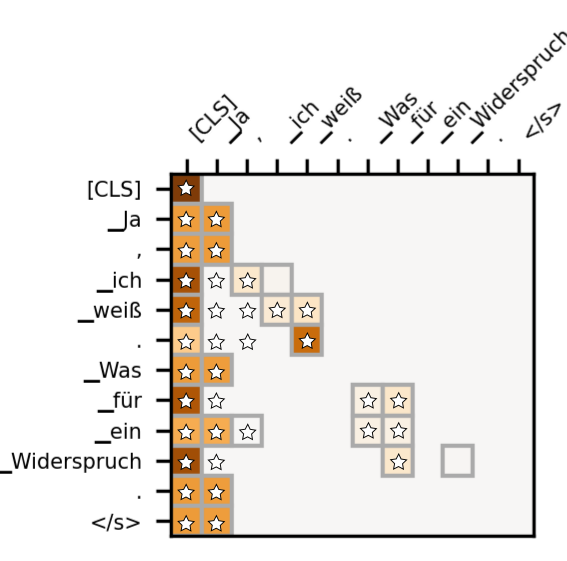}
    \caption{Learned patterns by Sparsefinder $k$-means (left) and the subsequent attention weights (right). Starred blocks represent ground-truth edges.} \label{fig:ex1_attention_map_supp}
\end{figure}

\begin{figure}[!htb]
    \centering
    \includegraphics[width=0.49\columnwidth]{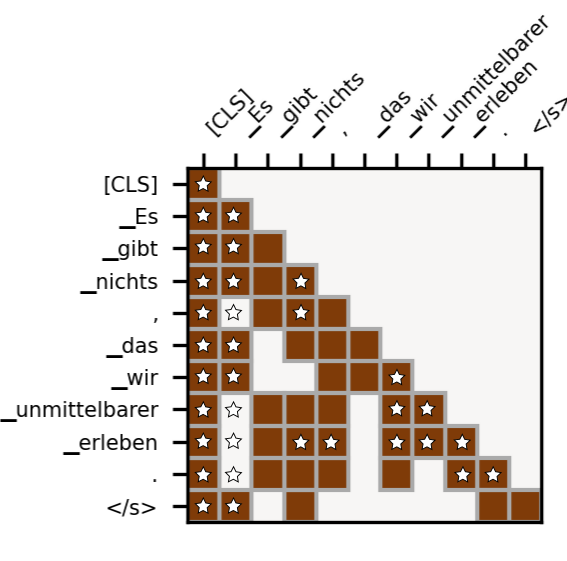}
    \includegraphics[width=0.49\columnwidth]{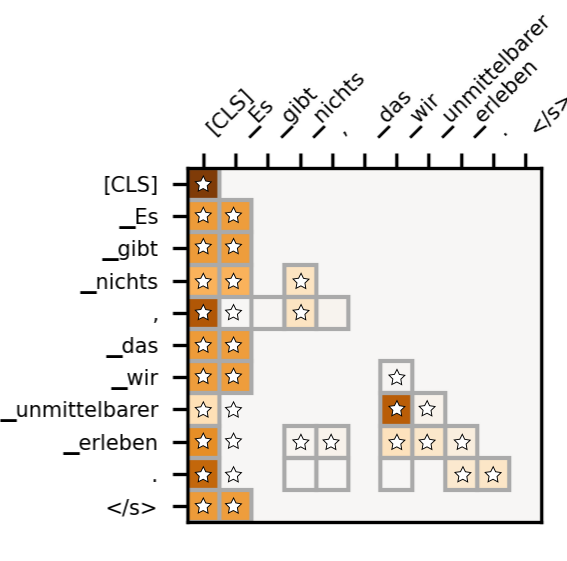}
    \caption{Learned patterns by Sparsefinder $k$-means (left) and the subsequent attention weights (right). Starred blocks represent ground-truth edges.} \label{fig:ex2_attention_map_supp}
\end{figure}

\section{Efficient Sparsefinder} \label{sec:efficient_sparsefinder_supp}

Plots for block size within \{1,2,4,8,16\} are shown in \ref{fig:sparsefinder_vs_bigbird_time_all_block_sizes}. For these experiments, we used a window size of 3 for all methods in order to better measure the impact of others hyper-parameters.

\begin{figure*}
    \centering
    \includegraphics[width=1\textwidth]{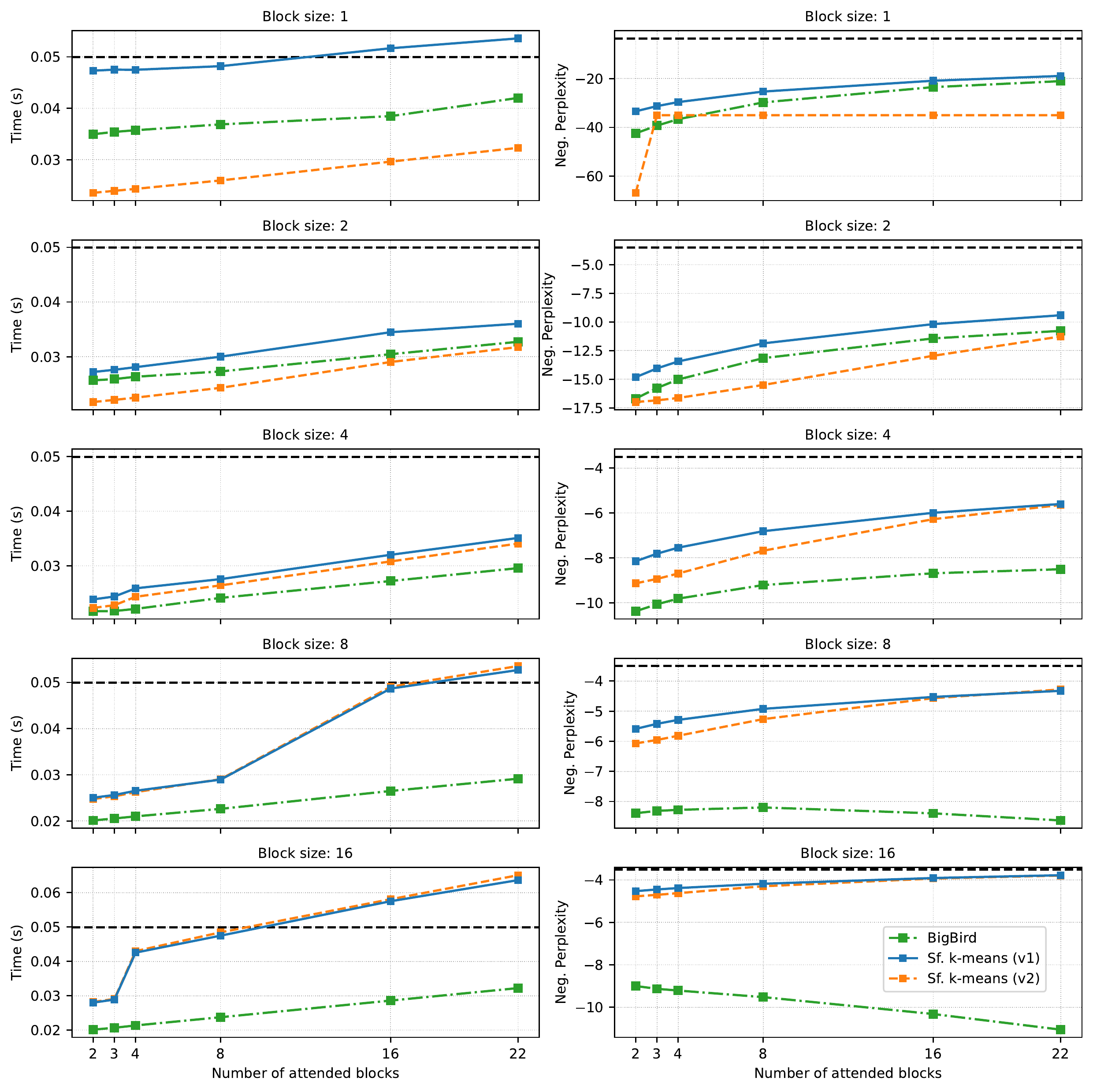}
    \caption{Comparison between Sparsefinder and BigBird in terms of running time and (negative) perplexity as a function of the number of random blocks for several block sizes. The horizontal dashed line represents the results obtained by the full $\alpha$-entmax transformer.}
    \label{fig:sparsefinder_vs_bigbird_time_all_block_sizes}
\end{figure*}

\end{document}